\documentclass{article}

\usepackage{microtype}
\usepackage{graphicx}
\usepackage{subfigure}
\usepackage{booktabs} % for professional tables
\usepackage{multirow}

\usepackage{hyperref}

\usepackage[accepted]{icml2024}

\usepackage{amsmath}
\usepackage{amssymb}
\usepackage{mathtools}
\usepackage{amsthm}
\usepackage{mathrsfs}

\usepackage[capitalize,noabbrev]{cleveref}

\theoremstyle{plain}
\newtheorem{theorem}{Theorem}[section]

\theoremstyle{definition}
\newtheorem{definition}[theorem]{Definition}

\theoremstyle{remark}

\def\trans{^{\mbox{\tiny{\sf T}}}}
\def\Span{\mathrm{Span}}
\newcommand{\indep}{\protect\mathpalette{\protect\independenT}{\perp}}
\def\independenT#1#2{\mathrel{\rlap{$#1#2$}\mkern2mu{#1#2}}}
\newcommand{\bA}{\mathbf{A}}
\newcommand{\bB}{\mathbf{B}}
\newcommand{\bU}{U}
\newcommand{\bX}{X}
\newcommand{\bx}{\mathbf{x}}
\newcommand{\sumi}{\sum_{i=1}^{n}}
\DeclareMathOperator*{\argmax}{argmax}
\DeclareMathOperator*{\argmin}{argmin}

\usepackage[textsize=tiny]{todonotes}
\newcommand{\rev}[1]{{#1}}

\icmltitlerunning{Enhancing Sufficient Dimension Reduction via Hellinger Correlation}

\begin{document}

\twocolumn[
\icmltitle{Enhancing Sufficient Dimension Reduction via Hellinger Correlation}



\icmlsetsymbol{equal}{*}

\begin{icmlauthorlist}
\icmlauthor{Seungbeom Hong}{sch}
\icmlauthor{Ilmun Kim}{sch2}
\icmlauthor{Jun Song}{sch}
% \icmlauthor{Firstname1 Lastname1}{equal,yyy}
% \icmlauthor{Firstname2 Lastname2}{equal,yyy,comp}
% \icmlauthor{Firstname3 Lastname3}{comp}
\end{icmlauthorlist}

\icmlaffiliation{sch}{Department of Statistics, Korea University, Seoul, South Korea}
\icmlaffiliation{sch2}{Department of Applied Statistics, Yonsei University, Seoul, South Korea}

\icmlcorrespondingauthor{Jun Song}{junsong@korea.ac.kr}

\icmlkeywords{Sufficient Dimension Reduction, Dependence Measure, Hellinger Correlation, f-divergence}

\vskip 0.3in
]



\printAffiliationsAndNotice{} % otherwise use the standard text.

\begin{abstract}
In this work, we develop a new theory and method for sufficient dimension reduction (SDR) in single-index models, where SDR is a sub-field of supervised dimension reduction based on conditional independence. Our work is primarily motivated by the recent introduction of the Hellinger correlation as a dependency measure. Utilizing this measure, we develop a method capable of effectively detecting the dimension reduction subspace, complete with theoretical justification. Through extensive numerical experiments, we demonstrate that our proposed method significantly enhances and outperforms existing SDR methods. This improvement is largely attributed to our proposed method's deeper understanding of data dependencies and the refinement of existing SDR techniques.
\end{abstract}
\section{Introduction}
\label{SE:Intro}
In the age of big data, the advent of high-dimensional datasets has transformed the landscape of statistical analysis and machine learning.
However, as the number of features in a dataset increases, so does the complexity of modeling and interpretation. 
This phenomenon, often referred to as the curse of dimensionality, poses a significant challenge to researchers and practitioners who seek to extract meaningful information and insights from vast and intricate data structures.

In response to this challenge, the field of sufficient dimension reduction (SDR) has emerged as a powerful approach to navigating high-dimensional space and uncovering the underlying structure without compromising interpretability. Much like how sufficient statistics provide essential information for estimation, sufficient dimension reduction methods furnish us with a subspace that contains adequate information to accurately estimate or explain the response variable. This approach is rooted in the idea that by identifying and preserving the key relationships in the notion of conditional independence as follows. For a univariate random variable $Y$ and a $p$-dimensional random vector $\bX$, the objective of linear sufficient dimension reduction is to seek out a matrix $\bB \in \mathbb{R}^{p\times d} \ (d < p)$ that possesses the smallest achievable column space such that 
\begin{equation}\label{def : linear SDR}
    Y \indep  \bX   \,|\,   \bB\trans \bX.
\end{equation}

It is important to note that the conditional independence does not change when $\bB$ is multiplied by any non-singular matrix $\bA \in \mathbb{R}^{d \times d}$. Thus, to make the target identifiable, the parameter needed to seek is the space spanned by the columns of $\bB$, i.e, $\Span({\bB)}$, not a matrix $\bB$ itself. The column space of $\bB$ with the smallest $d$ is called the \textit{central space} denoted by $\mathscr{S}_{Y|\bX}$ and the dimension of the central space is the structural dimension, say $d$. 

{The conditional independence \eqref{def : linear SDR} can also be represented as 
$$
Y \indep X \,|\, R(X),
$$
for a proper linear mapping, $R:\mathbb{R}^p \rightarrow \mathbb{R}^d$. It can be further represented as
$$
Y \,|\, X \sim Y \,|\, R(X),
$$
where $\sim$ means equal in distribution, which tells us that once $R(X)$ is identified, no more information about $Y$ can be obtained from $X$, and all the regression information for the predictor is preserved through $R(X)$. Another equivalent statement is 
$$ X \,|\, Y, R(X) \sim X \,|\, R(X). $$ Consider rewriting that
$X$ represents data $D$ and $Y$ symbolizes a parameter $\theta$. Under this reinterpretation, the above statement is equivalent to $D \,|\, (\theta,R) \sim D \,|\, R$, which suggests that $R$ acts as a sufficient statistic. Hence, the SDR mapping $R(\cdot)$ aligns with the traditional concept of statistical sufficiency. A key distinction, however, lies in the nature of the sufficient statistic versus the SDR: while a sufficient statistic is observable, the SDR involves parameters that require estimation. \citet{adragni2009sufficient} explains this conceptual idea of sufficiency more rigorously.}

If we have a proper $\bB$ satisfying conditional independence as in (\ref{def : linear SDR}), then the response $Y$ can be represented as
\begin{equation}\label{continuous Y SDR}
    Y = g(b_1\trans \bX, \ \dotsc, \ b_d\trans \bX, \varepsilon),
\end{equation}
where $g : \mathbb{R}^{p + 1} \mapsto \mathbb{R}$ is an unknown measurable function and $\varepsilon$ is independent of $\rev{\bX}$ with mean zero, and $b_1,\cdots,b_d$ are the columns of $\bB$. 
As there is no strong prerequisite for the function $g$, dimension reduction can be performed without relying on a specific model.

\rev{Various SDR methods have been proposed and successfully applied in diverse disciplines such as bioinformatics \citep{chiaromonte2002dimension,Hsueh2016}, finance \citep{wang2023quantitative}, marketing \citep{naik2000new} and ecology \citep{roley2008predicting}. The two most dominating approaches for SDR are inverse regression and forward regression approach.} The inverse method requires an additional assumption on the conditional distribution of the predictors given the response, $\bX \, | \, Y$. 
The well-known techniques for the inverse method include sliced inverse regression (SIR) \cite{Li1991SlicedReduction}, sliced average variance estimation (SAVE) \cite{cookSlicedInverseRegression1991}, and directional regression (DR) \cite{Li2007OnReduction}. The forward method requires less assumptions on the conditional distribution but has additional assumptions on the link function. 
The minimum average variance estimation (MAVE) \cite{Xia2002AnSpace} is a widely recognized approach to the forward method. We refer to \citet{Li2018Sufficient} for a detailed explanation. %Although the prior methods have been shown to be effective, their success relies on specific assumptions: inverse regression methods require an elliptical distribution of the predictors $X$, while forward-regression approaches require a condition on the link function $g(\cdot)$, which is difficult to examine with the dataset.  
Although the prior methods have proven effective, their success relies on specific assumptions mentioned earlier: inverse regression methods impose strong assumptions on the conditional/marginal distribution of $X$ such as the predictors $X$ following an elliptical distribution, while forward-regression approaches require an assumption on the link function $g(\cdot)$. All these assumptions are difficult to verify with the dataset.  

On the other hand, there are other directions of SDR studies, which build on measures of statistical dependence. Quantifying the dependence between two random objects has been a central topic in statistics.
%Dependency measures characterize statistical independence without relying on parametric assumptions. 
Some notable examples of dependence measures, especially in nonparametric settings, include the distance covariance \citep{Szekely2007MeasuringDistances, Szekely2009BrownianCovariance}, the Hilbert--Schmidt independence criterion (HSIC) \citep {Gretton2005MeasuringNorms} and the ball covariance \citep{Pan2020BallSpace}. The SDR subspace can be found by maximizing the dependence between the response $Y$ and predictors $\bX$. \citet{Sheng2013DirectionCovariance} showed that sufficient dimension reduction can be achieved for a single-index model by the distance covariance. \citet{Sheng2016SufficientCovariance} extended this method to a general structural dimension $d$. Similarly, \citet{Zhang2015DirectionCriterion} proposed a way to utilize the HSIC for the single index model, and \citet{Xue2017SufficientCriterion} extended it to a general dimension $d$. In addition, \citet{Zhang2019RobustCovariance} analyzed single and multi-index models based on the ball covariance. See \citet{Dong2021AOptimization} for a comprehensive review.

While existing SDR methods based on dependence measures have shown their effectiveness, they are not free from limitations. %While dependency measure-based SDR methods have demonstrated their efficacy, they are not free from limitations. 
One notable limitation is in the interpretation of the dependence measure itself: while a zero value indicates independence between two random variables, a larger measure does not necessarily imply a stronger relationship between them. This gap can lead to potential misinterpretations when assessing the strength of these relationships. Moreover, these methods encounter challenges regarding their theoretical foundation. \rev{Some dependence measure-based SDR methods either lack comprehensive theoretical validation or rely on specific independence assumptions that are difficult to verify in practical applications. It raises concerns about their reliability.}

\rev{To address these issues, this article introduces a new SDR method using Hellinger correlation, which improves the accuracy of estimating the central subspace while achieving its theoretical justification with weaker assumptions than existing methods. More precisely, compared to other dependence measures, the Hellinger correlation is more adept at capturing the strength of various relationships and satisfies the natural axioms for dependence measures \citep{Geenens2022TheCorrelation}. Equipped with the benefits of the Hellinger correlation, our SDR approach is straightforward to implement and consistently delivers significant improvements over existing methods in almost all cases considered. Moreover, we distinguish the proposed method from existing approaches by establishing a theoretical foundation, including consistency guarantees with minimal assumptions.

%To focus on the foundational establishment of a new SDR method with a dependency measure, we focus on the single-index model in which the target dimension equals one. Although extending our method to multi-index models is feasible and promising as done in \citet{christou2020central} and  \citet{Sheng2016SufficientCovariance}, we have opted to prioritize foundational principles for SDR. This approach ensures clarity and maintains the robustness of the article.}

To establish a solid foundation for a new SDR method based on a dependence measure, we focus on the single-index model where the target dimension is one. Although extending our method to multi-index models is feasible and promising, as discussed in Section \ref{SE : Discuss} and other work \citep{Sheng2016SufficientCovariance,christou2020central}, we have opted to prioritize the foundational principles of SDR by concentrating on the simple yet fundamental single-indexed model. This approach aims to ensure clarity and maintain the robustness of the article, while laying a strong groundwork for future expansions. 
}

%The computational algorithm requires an initial direction for iteration. Our findings indicate that a random direction can effectively detect the central space. 
%Our empirical results demonstrate a marked improvement over existing methods in almost all cases that we consider. 
%However, when other existing SDR methods are employed, our results demonstrate a marked improvement in almost all cases. Our simulation studies show that our method significantly enhances the performance of existing methods. In summary, our proposed method necessitates fewer assumptions than other methods while offering superior performance compared to current existing SDR methods.

The remainder of this article is organized as follows. Section \ref{SE:Back and Moti} provides background information and motivation for this paper. Section \ref{SE : TR} delves into the optimization methods and presents the theoretical results. In Section \ref{SE : NR}, we provide simulation results that compare our method with existing ones. Section \ref{sec: real data} presents a real data application of the method. Section \ref{SE : Discuss} summarizes our contributions and discusses several directions for future work. Lastly, the appendix includes additional simulation results. The code that implements our proposed method is available at \href{https://github.com/JSongLab/SDR_HC}{https://github.com/JSongLab/SDR\_HC}.

\section{Background and Motivation}
\label{SE:Back and Moti}

%There have been studies that make new dependence measures that zero implies independence without normal assumptions. 
%The distance covariance (\cite{Szekely2007MeasuringDistances}, \cite{Szekely2009BrownianCovariance}), the Hilbert-Schmidt independence criterion \cite{Gretton2005MeasuringNorms} and the ball covariance \cite{Pan2020BallSpace} are the results. 
% Several studies have introduced novel dependence measures that fully characterize statistical independence without relying on parametric assumptions, such as the distance covariance \citep{Szekely2007MeasuringDistances,Szekely2009BrownianCovariance}, the Hilbert-Schmidt independence criterion \citep {Gretton2005MeasuringNorms} and the ball covariance \citep{Pan2020BallSpace}.
% These measures provide 

% Recently, \citet{Geenens2022TheCorrelation} introduce the Helilnger Correlation 
% To our knowledge, the Hellinger correlation catches the dependence between two random variables well than any other existing measures.
% Also, it has invariant property under the monotonic transformation. 
% Thus, we concluded that the Hellinger correlation can be used to recover the central subspace.
% In the following, we provide some background to explain our findings. 

\subsection{SDR through dependence measure}

% \textcolor{red}{[Repeated] Pearson's correlation has an expandable property; 
% If the correlation is zero, both random variables are independent when both variables are normal. There have been studies that make new dependence measures that zero implies independence without normal assumptions. 
% The distance covariance \citep{Szekely2007MeasuringDistances,Szekely2009BrownianCovariance}, the Hilbert-Schmidt independence criterion \citep{Gretton2005MeasuringNorms} and the ball covariance \citep{Pan2020BallSpace} are the results. These measures can be used for sufficient dimension reduction with some additional assumptions. The below briefly summarize the distance covariance-based SDR methods\citep{Sheng2013DirectionCovariance, Sheng2016SufficientCovariance} which are mostly well-known. For a more comprehensive review, see \citet{Dong2021AOptimization}.}
As mentioned in Section \ref{SE:Intro}, there have been studies focused on developing SDR methods that utilize dependence measures. %The following section briefly summarizes the distance covariance-based SDR methods \citep{Sheng2013DirectionCovariance, Sheng2016SufficientCovariance}, which are among the most well-known in this field. \textcolor{red}{this section does not really focus on distance covariance?}
To explain the idea, assume that the structural dimension is known as $d$. 
Let $\rev{\boldsymbol{\eta}_0}$ be the basis of the central subspace and $\rev{\boldsymbol{\eta}_1}$ be the basis of the orthogonal complement of the central subspace. 
In other words, $\rev{\boldsymbol{\eta}} = (\rev{\boldsymbol{\eta}_0},  \rev{\boldsymbol{\eta}_1})^\top$ is a basis of $\mathbb R^p$. 
Let $\rho: \mathbb R^d \times \mathbb {R} \to [0, \infty)$ be a generic dependence measure between two random quantities. A SDR method based on $\rho$ seeks to find the central space $\mathscr{S}_{Y|\bX}$ by solving the following optimization problem: 
\begin{equation*}
    \bB_0 = \argmax_{\bB \in \mathbb{R}^{p\times d}} \rho(\bB\trans\bX,Y) \quad \text{   subject to } \quad  \bB\trans \Sigma_{\bX} \bB = I_d 
\end{equation*}
where $\Sigma_{\bX}$ is the covariance matrix of $\bX \in \mathbb{R}^p$. 
To conclude that the maximizer recovers the central space, i.e., $\Span(\bB_0) = \Span(\rev{\boldsymbol{\eta}_0})$, %the additional assumption is needed. 
%\begin{equation*}
%    \rev{\boldsymbol{\eta}_0}\trans \bX \ \indep \ \rev{\boldsymbol{\eta}_1}\trans \bX
%\end{equation*} 
%The assumption requires independence between projected random variables. However, it is not easily examined from the dataset without the normality assumption. 
the previous approaches impose an additional independence assumption on projected random variables, namely
\begin{equation*}
    \rev{\boldsymbol{\eta}_0}\trans \bX \ \indep \ \rev{\boldsymbol{\eta}_1}\trans \bX,
\end{equation*} 
which is not easily verifiable in practical applications. In contrast, we aim to remove this additional restriction and propose a more reliable SDR method.

\subsection{Copula}
%The copula is an essential tool in high-dimensional analysis, enabling us to estimate random vectors through the estimation of marginal distributions. This tool has applications across various fields, with finiance being a prime example of its extensive usage. In finance, copulas allow researchers to model complex dependencies between multiple assets, and provide valuable insights into understanding the dynamics of financial markets.
We next briefly discuss copulas. Copulas are essential tools in high-dimensional analysis, enabling us to estimate random vectors through the estimation of marginal distributions. This tool has applications across various fields, with finance being a prime example of its extensive usage \citep{cherubini2004copula}. Assume that there exists a continuous random vector $\rev{\bX} = (X_1, X_2,\dotsc, X_p)\trans$. Let $U_i =  F_i(X_i)$ where $F_i$ is the cumulative distribution function of $X_i$. 
By the integral probability transform, all $U_i$ are uniform random variables on the interval $[0, 1]$. 
The copula $C$ is defined as 
\begin{equation}\label{def : copula}
    C(u_1,\dotsc,u_p) = P(U_1 \le u_1,\dotsc, U_p \le u_p).  
\end{equation}
In other words, the copula $C$ is the joint cumulative distribution function of a random vector in the unit cube, where each marginal is a uniform random variable.

\citet{Sklar1959FonctionsMarges} explains that a copula is an adequate tool for understanding the distribution of a random vector, and establishes the following result:
\begin{theorem}[Sklar, 1959]
    Let $\rev{\bX} = (X_1, X_2,\dotsc, X_p)\trans$ be a random vector. Suppose that $F$ and $f$ are the joint cumulative distribution function and the joint probability density function of $\rev{\bX}$. Then, there exists a function $C : [0, 1]^p \to [0, 1]$, called the copula of $\rev{\bX}$, such that 
    \begin{equation*}
        F(x_1,x_2,\dotsc,x_p) = C(F_1(x_1),F_2(x_2,),\dotsc,F_p(x_p)).
    \end{equation*}
    Additionally, there exists a function $c : [0, 1]^\rev{p} \to [0, \infty)$, called the copula density of $\rev{\bX}$, such that 
    \begin{equation*}
        \begin{split}
            f(x_1,x_2,\dotsc,x_p) &= c(F_1(x_1),F_2(x_2,),\dotsc,F_p(x_p)) 
        \\&\times f_1(x_1)f_2(x_2) \dotsb f_p(x_p).    
        \end{split}
    \end{equation*}
\end{theorem}

For a random variable $Y$, we have $F_Y(y) = C(F_Y(y))$ and $f_Y(y) = c(F_Y(y))f_Y(y)$. 
Thus, the univariate random variable copula is the identity function and the density of the copula is 1 in the unit interval [0, 1]. 
This property will be used later in the construction of the proposed method.

Since the cumulative distribution function of a random variable is a monotonic function, a copula is invariant to the monotonic transformation of its marginals. 
This property is powerful in determining the dependence between random variables.

\subsection{f-divergence}

The $f$-divergence is a function that measures the difference between two distributions $P$ and $Q$ given as 
\begin{equation}\label{def of f-divergence}
    \begin{split}
        D_\varphi(P\|Q) =  \int \varphi\left(\frac{dP}{dQ}\right)dQ,    
    \end{split}
\end{equation}
where $\varphi : (0, \infty) \to \mathbb R$ is convex and $\varphi(1) = 0$.
The $f$-divergence family encompasses a wide range of statistical distances between distributions. Some notable examples include the Kullback--Leibler divergence with $\varphi(t) = t\log t$, the squared Hellinger distance with $\varphi(t) = (\sqrt{t} -1)^2$ and the total variation distance with $\varphi(t) = |t - 1|/2$.

If $\varphi$ is strictly convex, $P$ and $Q$ are identical distributions if and only if $D_\varphi(P\|Q) = 0$. 
This characteristic property allows us to build upon the $f$-divergence to test the independence between two random vectors. More formally, the hypotheses for independence testing are given as
\begin{equation*}
    H_0 : F_{XY} = F_XF_Y \quad \text{versus} \quad  H_1 : F_{XY} \ne F_XF_Y,
\end{equation*}
where $F_{XY}$ is the joint distribution of $(X,Y)$ and $F_XF_Y$ is the product of the marginal distributions. 
%The $f$-divergence between $F_{XY}$ and $F_XF_Y$ can be a candidate quantity to test independence. 
The characteristic property of the $f$-divergence implies that the above hypotheses can be equivalently written as $H_0: D_\varphi(F_{XY}\|F_XF_Y) =0$ versus $H_1: D_\varphi(F_{XY}\|F_XF_Y) \neq 0$. Hence one can use an estimator of
\begin{equation*}%\label{ind test f-div}
    D_\varphi(F_{XY}\|F_XF_Y) = \iint \varphi\left(\frac{dF_{XY}}{dF_XdF_Y}\right)dF_XdF_Y
\end{equation*}
as a test statistic for independence testing. %$D_\phi = 0$ implies that the two distributions $F_X$ and $F_Y$ are independent. 
If both $F_X$ and $F_Y$ are absolutely continuous, then $D_\varphi$ can be represented in terms of density functions as follows:
\begin{equation}\label{f-div by pdf}
    D_\varphi(F_{XY}\|F_XF_Y) = \iint f_Xf_Y\varphi\left(\frac{f_{XY}}{f_Xf_Y}\right)dxdy. 
\end{equation}
\citet{Kinney2014EquitabilityCoefficient} proved that any measure of dependence constructed by $f$-divergence \eqref{f-div by pdf} holds data processing inequality. That is, $D(X, Z) \le D(Y, Z)$ where $X\indep Z \, | \, Y$.

\subsection{Hellinger Correlation}

\citet{Geenens2022TheCorrelation} proposed the Hellinger correlation as a tool for capturing the dependence between a pair of random variables. 
As implied by its name, it is formulated based on the squared Hellinger distance, which is an example of the $f$-divergence. 
%As discussed earlier, the $f$-divergence between the joint distribution and the product of the marginal distributions is a quantity that represents the magnitude of dependence.
%For convenience of estimation, the Hellinger correlation is constructed by the density of the copula.
%Since the Hellinger correlation has been restricted to a univariate variable to date, $c_Xc_Y = 1$ on the unit square $\mathcal I^2$. 
%Focusing on univariate random variables $X$ and $Y$, 
More specifically, the squared Hellinger distance between $F_{XY}$ and $F_XF_Y$ is given as
\begin{equation}%\label{Squared Helleniger distance}
\begin{split}
    \mathcal H^2(X, Y) &= \frac{1}{2}\iint_{\mathbb R^2}
    \left(\sqrt\frac{dF_{XY}}{dF_XdF_Y} - 1\right)^2dF_XdF_Y\\
    &=\frac{1}{2}\iint_{\mathcal I^2} (\sqrt{c_{XY}(u_x, u_y)} - 1)^2du_xdu_y \\
    &= 1 - \iint_{\mathcal I^2} \sqrt{c_{XY}(u_x, u_y)}du_xdu_y \\ 
    &:= 1 - \mathcal{B}(X, Y), \label{def:bhattacharyya}
\end{split}
\end{equation}
where $\mathcal I^2$ denotes the unit square $[0,1]^2$. In the above equations, $c_{XY}$ denotes the joint copula density of $U_X$ and $U_Y$ where $U_X$ and $U_Y$ denote the cumulative distribution function of $X$ and $Y$, respectively.
The quantity $\mathcal{B}(X, Y)$ in the last line is referred to as the Bhattacharyya affinity coefficient \citep{bhattacharyya1943measure} between the copula densities. %, which can be interpreted as $\mathrm{E}[c_{XY}(U_X, U_Y)^{-1/2}]$ \ik{check: $\mathrm{E}[c_{XY}(U_X, U_Y)^{1/2}]$}
%This expectation is related to the Renyi and Tsallis entropy, which will be discussed later. 
From now on, we will write $\mathcal B(X,Y)$ as $\mathcal B$ for simplicity.

To motivate the Hellinger correlation, consider a bivariate normal random vector $(X, Y) \sim N((0,0) , (\begin{smallmatrix}
     1 & \rho \\ \rho & 1
 \end{smallmatrix}))$. As dicsussed in \citet{Geenens2022TheCorrelation}, the squared Hellinger distance between $X$ and $Y$ has an explicit expression as
\begin{equation*}
    \mathcal H^2(X, Y) = 1 - (2(1-\rho^2)^{1/4}) / (4 - \rho^2)^{1/2}
\end{equation*}
or $\mathcal B = (2(1-\rho^2)^{1/4}) / (4 - \rho^2)^{1/2}$ in terms of the Bhattacharyya affinity coefficient. As a result, the correlation parameter $\rho$ can be written as 
\begin{equation*}
    |\rho| = \frac{2}{\mathcal B^2} \{\mathcal B^4 + (4 - 3\mathcal B^4)^{1/2} - 2\}^{1/2}. 
\end{equation*}
This relationship leads to the Hellinger correlation between random variables $X$ and $Y$ defined as follows.% defined as
%\begin{equation}\label{HCor definition}
%    H(X, Y) = h(\mathcal B) =  \frac{2}{\mathcal B^2} \{\mathcal B^4 + (4 - 3\mathcal B^4)^{1/2} - 2\}^{1/2}. 
%\end{equation}
\rev{
\begin{definition} \label{def:hellinger correlation} Let $\mathcal{B}$ denote the Bhattacharyya affinity coefficient for $(X,Y)$ defined in \eqref{def:bhattacharyya}. The Hellinger correlation between $X$ and $Y$ is defined as 
\begin{equation}\label{HCor definition}
    H(X, Y) = \frac{2}{\mathcal B^2} \{\mathcal B^4 + (4 - 3\mathcal B^4)^{1/2} - 2\}^{1/2}. 
\end{equation}
\end{definition}
}
By construction, the Hellinger correlation $H(X,Y)$ coincides with the Pearson correlation when $(X,Y)$ follows a joint normal distribution, whereas they can differ significantly for non-normal distributions. It is worth noting that the properties of $f$-divergence and copula are preserved in the Hellinger correlation as the map $h:[0,1] \to [0,1]$, given as $h(x) = 2x^{-2} \{x^4 + (4 - 3x^4)^{1/2} - 2\}^{1/2}$, is a bijection.
%if $(X, Y)$ follows a joint normal distribution then $H(X, Y)$ is the same with the absolute value of Pearson correlation. Since the mapping $h:[0,1] \to [0,1]$ is a bijection, the property of $f$-divergence and copula is preserved in Hellinger correlation.

The Hellinger correlation has several attractive properties, worth highlighting. First of all, like distance correlation, the Hellinger correlation fully characterizes independence, i.e., $H(X, Y) = 0$ if and only if $X$ and $Y$ are independent. However, unlike distance correlation, the Hellinger correlation does not depend on any moment conditions. Moreover, it is properly normalized as $0 \leq H(X,Y) \leq 1$, and equals one when $X$ and $Y$ are deterministically predictable from each other. See (P6) in \citet{Geenens2022TheCorrelation} for a more precise statement. The latter property is in sharp contrast to other popular measures such as Pearson's correlation, distance correlation and rank-based correlations. In particular, Pearson's correlation and distance correlation are 1 if a random variable is an affine transformation of the other variable. %For example, if $Y = aX + b$, the Pearson correlation and the distance correlation between $X$ and $Y$ are 1. 
Additionally, rank-based measures such as Spearman's $\rho$, Kendall's $\tau$, and Hoeffeding's $D$ are 1 if two random variables have a monotonic deterministic relationship. 
More fundamentally, the Hellinger correlation takes 1 if and only if there exists a Borel function $\Phi : [0,1] \to \mathbb R^2$ such that $(X, Y) = \Phi(U)$ where $U$ is a uniform random variable in the interval $[0,1]$. 
%This property comes from the choice of $\varphi$ for $f$-divergence. 
%Let $\varphi_0 = \lim_{t\to 0}\varphi(t)$ and $\varphi_0^* = \lim_{t\to 0}t\varphi(1/t)$. 
%If $\varphi_0 + \varphi_0^*$ is finite, the measure of dependence based on $f$-divergence has the property. 
Another important property of the Hellinger correlation is that it is invariant to any monotonic transformations. This means that for any two strictly monotonic functions $\psi_1, \psi_2$, the following relationship holds
\begin{equation*}%\label{HCOR : monotinic invariance}
    H(\psi_1(X), \psi_2(Y)) = H(X, Y).
\end{equation*} 
This invariance property has been highlighted as a fundamental property of any valid dependence measure. We refer the reader to \citet{Geenens2022TheCorrelation} for further discussion on the properties of the Hellinger correlation. We also point out that the Hellinger correlation tends to be more sensitive to non-linear and realistic dependence than other popular dependence measures as illustrated in the simulation section in \citet{Geenens2022TheCorrelation}. 

%\rev{ as well as empirical results, which demonstrates that the Hellinger correlation is more sensitive to non-linear and realistic dependence than other popular dependence measures, such as distance correlation.}
%This property comes from the invariance property of the copula, which we discussed earlier. 
%Since the Hellinger correlation is constructed by the integral of the copula, the Hellinger correlation is also invariant.

Definition \ref{def:hellinger correlation} indicates that estimating $\mathcal B$ is sufficient for estimating $H(X, Y)$.
\citet{Geenens2022TheCorrelation} introduce an estimator of $\mathcal B$ based on the estimator of \citet{Leonenko2008ADensities}. To explain, let $\{(X_1, Y_1),\dotsc, (X_n, Y_n)\}$ be a random sample of size $n$. Let $\rev{\bU_i} = (U_{X_i}, U_{Y_i}) = (F_X(X_i), F_Y(Y_i))$. Under the continuity assumption for $F_X$ and $F_Y$, it is clear that $U_{X_i}$ and $U_{Y_i}$ are uniform random variables. Let $\rev{\hat{\bU}_i}$, the sample version of $\rev{\bU_i}$, be $(\hat F_X(X_i), \hat F_Y(Y_i))$ where $\hat F_X(u) := (1 / (n + 1)) \sum_{i=1}^n I_{\{X_i \le u\}} $ and $\hat F_Y(u) := (1 / (n + 1)) \sum_{i=1}^n I_{\{Y_i \le u\}}$. Let $R_i = \min_{j \ne i}\|\rev{\bU_j} - \rev{\bU_i}\|_2$ and $\hat R_i = \min_{j \ne i}\|\rev{\hat{\bU}_j} - \rev{\hat{\bU}_i}\|_2$. Then the final estimator of $\mathcal B$ suggested by \citet{Geenens2022TheCorrelation} is 
\begin{equation}\label{estimator for B}
    \hat{\mathcal{B}}_n = \frac{2\sqrt{n - 1}}{n} \sumi \hat R_i,
\end{equation}
and the corresponding estimator of the Hellinger correlation is 
\begin{equation}\label{estimator for H}
    \hat H_n(X, Y) = \frac{2}{\hat{\mathcal B}^2_n} \{\hat{\mathcal B}^4_n + (4 - 3\hat{\mathcal B}^4_n)^{1/2} - 2\}^{1/2}.
\end{equation}
Our work builds upon these estimators of the Bhattacharyya affinity coefficient and the Hellinger correlation, and proposes \rev{an} SDR method that offers both theoretical and empirical advantages over existing approaches.

\rev{There are several instances of $f$-divergence (e.g., the total variation distance) that share similar properties as the Hellinger correlation. However, unlike the Hellinger correlation, there is currently a lack of computationally efficient estimators with solid theoretical guarantees for these divergences, which is the main bottleneck for using those in SDR applications. We therefore focus on the Hellinger correlation in this work, while leaving the exploration of other $f$-divergence measures for SDR as an interesting avenue for future research.}

%\textcolor{red}{This paragraph does not add much?}The estimator of \citet{Leonenko2008ADensities}, which introduces a method for estimating entropy, is based on the nonparametric density estimation by nearest neighbor distance method. 
%If one of our variables of interest is a random vector, we need to estimate two density functions simultaneously. 
%It is numerically unstable. 
%If both variables of interest are random vectors, then we need to estimate three density functions. 
%Thus, there has been no consistent estimator for the Hellinger correlation between multivariate random vectors. 

\section{Main Results}
\label{SE : TR}
%As mentioned earlier, we assume that the structural dimension is $d=1$ and both $\bX$ and $Y$ have continuous distributions. 
We now introduce the main results of this work by focusing on the setting where the structural dimension is $d=1$ and both $\bX$ and $Y$ have continuous distributions. In this setting, the multi-index model \eqref{continuous Y SDR} becomes the single-index model
\begin{equation}\label{Model Assumption}
    Y = g(\rev{\eta_0}\trans \bX,\varepsilon),
\end{equation}
%where $Y$, the response, is a univariate random variable, $\bX$, the predictor, is a $p-$dimension random vector, $g$ is some unknown measurable function and $\varepsilon$ is an unknown error. $\varepsilon$ is independent of $\bX$ and has zero mean. 
%where $Y \in \mathbb{R}$ represents the response, $\bX \in \mathbb{R}^p$ is the predictor, $g$ is an unknown link function, and $\varepsilon$ is the error term. We assume that $Y$, $\bX$, $g$ and $\varepsilon$ in this single index model meet the conditions required for the general model in \eqref{continuous Y SDR}. 
%and we assume that $Y$, $\bX$, $g$ and $\varepsilon$ meet the conditions required for the general model in \eqref{continuous Y SDR}. 
and our goal is then to estimate the central subspace spanned by $\rev{\eta_0}$ through the Hellinger correlation. %by finding the basis of the central subspace, $\rev{\eta_0}$, by maximizing the Hellinger correlation.

\subsection{Method}
For the purpose of identification, we restrict the parameter space to the unit sphere of $\mathbb R^p$, which is denoted as $\mathbb S^{p-1}$. Since $H$ is a monotonically decreasing function of $\mathcal B$, minimizing $\mathcal B$ is equivalent to maximizing $H$.  
%The reason we choose to minimize $\mathcal B$ is that there is an efficient estimator for $\mathcal B$ computed by the nearest neighbor method. 
Thus, our objective is to find $\rev{\eta_0}$ such that
 \begin{equation}\label{objective function}
     \rev{\eta_0} = \argmax_{\rev{\alpha}  \in \mathbb S^{p-1}} H(\rev{\alpha} \trans \bX, Y) =  \argmin_{\rev{\alpha}  \in \mathbb S^{p-1}} \mathcal B(\rev{\alpha} \trans \bX, Y). 
 \end{equation}
We use the sphere coordinate to represent $\mathbb S^{p-1}$. 
To represent the direction vector in the Euclidean coordinate, we convert it to the $p -1$ radian tuple. 
For $\rev{\alpha}  \in \mathbb R^p$, there exists $\rev{\phi} = (\phi_1, \phi_2, \dotsc, \phi_{p-1})$ where $\phi_1,\dotsc,\phi_{p-2} \in [0, \pi]$ and $\phi_{p-1} \in [0, 2\pi)$ defined as below: 
\begin{align}\label{to sphere coordinate}
\begin{split}
    \phi_1 &= \arctan\Big(\sqrt{\alpha_p^2 + \dotsb + \alpha_2^2} / \alpha_1\Big)\\
    \phi_2 &= \arctan\Big(\sqrt{\alpha_p^2 + \dotsb + \alpha_3^2} / \alpha_2\Big)\\
            &\vdots\\
    \phi_{p-1} &= \arctan(\alpha_p / \alpha_{p-1}).
\end{split}
\end{align}
Given the radian tuple, our optimization process consists of two steps.
First, we use the simulated annealing method \cite{belisle1992convergence} given initial values produced by existing SDR methods: SIR, SAVE, DR, and MAVE.
Second, starting with the results of the first method, we apply the downhill simplex method proposed by \citet{Nelder1965AMinimization}. After optimization, we transform $\rev{\phi}$ and return $(\alpha_1,\ldots,\alpha_p)^\top \in \mathbb{R}^p$ defined as 
\begin{align}\label{from sphere coordinate}
\begin{split}
    \alpha_1 &= \cos(\phi_1)\\
    \alpha_2 &= \sin(\phi_1)\cos(\phi_2)\\
            &\vdots\\
    \alpha_{p-1}  &= \sin(\phi_1)\dotsb\sin(\phi_{p-2})\cos(\phi_{p-1})\\
    \alpha_p &= \sin(\phi_1)\dotsb\sin(\phi_{p-2})\sin(\phi_{p-1}).
\end{split}
\end{align}

The sample-level estimation procedure is based on our estimators for $\mathcal{B}$ and $H$ given in (\ref{estimator for B}) and (\ref{estimator for H}):
\begin{align*}
     \rev{\hat{\eta}_n} = \argmax_{\rev{\alpha}  \in \mathbb S^{p-1}} \hat{H}_n(\rev{\alpha} \trans \bX, Y) =  \argmin_{\rev{\alpha}  \in \mathbb S^{p-1}} \hat{\mathcal B}_n(\rev{\alpha} \trans \bX, Y),
\end{align*}
and our next goal is to investigate theoretical and empirical properties of $\rev{\hat{\eta}_n}$. 

\rev{Before moving on, let us briefly discuss the computational complexity of the proposed procedure. Our method involves computing the Hellinger correlation estimator, which has a complexity of $O(n^2 p)$. The transformation of $\alpha$ into spherical coordinates and back into Euclidean coordinates adds a complexity of $O(p)$ per iteration, maintaining the overall complexity at $O(n^2 p)$ per iteration. The number of iterations for the downhill simplex method to reach a local optimum varies depending on several factors such as initial values and tolerance, making precise complexity analysis challenging. Nevertheless, denoting the number of iterations as $k$, the overall complexity of our method can be concisely written as $O(n^2 p k)$.}

\subsection{Theoretical Results}
In this section, we show the consistency of the sample-level estimation $\rev{\hat{\eta}_n}$. %The optimization is based on the Hellinger correlation which is difficult to manipulate due to its complicated definition. 
First, we show that the population-level estimation (\ref{objective function}) recovers the central space and the solution is unique up to a sign-flip. Specifically, the next theorem shows that we can recover the central subspace by maximizing $H(\rev{\alpha} \trans \bX, Y)$ with respect to $\rev{\alpha}  \in \mathbb S^{p-1}$, i.e., $\rev{\eta_0} = \argmax H(\rev{\alpha} \trans \bX, Y)$ over all $\rev{\alpha}  \in \mathbb S^{p-1}$.
\begin{theorem}\label{argmin is base of central space}
    Let $\bX \in \mathbb R^p$ be a random vector and $Y \in \mathbb R$ be a random variable. 
    Let $\rev{\eta_0}\in \mathbb S^{p-1}$ be the basis of the central subspace $\mathscr S_{Y|\bX}$. 
    Then $\rev{\eta_0} = \argmax H(\rev{\alpha} \trans \bX, Y)$ for all $\rev{\alpha}  \in \mathbb{S}^{p-1}$. Moreover it holds that $H(\rev{\eta_0}\trans \bX, Y) = H(\rev{\alpha} \trans \bX, Y)$ if and only if $\Span(\rev{\eta_0}) = \Span(\rev{\alpha} )$.
\end{theorem}
\begin{proof}
    \citet{Geenens2022TheCorrelation} explained that the Hellinger correlation satisfies the generalized data processing inequality. That is, $H(X, Y)  \le \min\{H(X, Z) , H(Y, Z)\}$ if $X \ \indep \ Y  \ | \ Z$.
    Let $P_{\rev{\alpha}}$ be the orthogonal projection matrix generated by $\rev{\alpha} $. 
    In addition, let $\sigma(A)$ denote the smallest $\sigma$-algebra generated by the random variable $A$.
    From sufficient dimension reduction assumptions, for any $\rev{\alpha}  \in \mathbb S^{p-1}$,
    \begin{equation*}
        \bX \ \indep \ Y \ | \ \rev{\eta} _0\trans \bX \Rightarrow P_{\rev{\alpha} } \bX \ \indep \ Y \ | \ \rev{\eta_0}\trans \bX
    \end{equation*}
    since $\sigma(P_{\rev{\alpha} } \bX) \subseteq \sigma(\bX)$.
    We also have $\sigma(P_{\rev{\alpha} } \bx) = \sigma(\rev{\alpha} \trans \bx)$.
    Thus,
    \begin{equation*}%\label{eq: x indep y alpha indep} 
        \rev{\alpha} \trans \bX \ \indep \ Y \ | \ \rev{\eta_0}\trans \bX.
    \end{equation*}
    By the property of the Hellinger correlation,
    \begin{equation*}
        H(\rev{\alpha} \trans \bX, Y) \le H(\rev{\eta_0}\trans \bX, Y).
    \end{equation*}
    Thus, $H(\rev{\alpha} \trans \bX, Y)$ achieves the maximum when $\rev{\alpha}  = \rev{\eta_0}$.

    Next we prove that $H(\eta_0 \trans \bX, Y) = H(\alpha \trans \bX, Y)$ \emph{if and only if} $\Span(\eta_0) = \Span(\alpha)$. The ``\emph{if}'' direction is trivial because the Hellinger correlation is invariant to monotonic transformations. We thus focus on the ``\emph{only if}'' direction. 
    
Suppose now that $\rev{\alpha_0}  \in \mathbb S^{p-1}$ is another maximizer. 
If there is a monotonic relation between $\rev{\alpha_0} \trans \bX$ and $\rev{\eta_0}\trans \bX$, then $\sigma(\rev{\alpha_0} \trans \bX)=\sigma(\rev{\eta_0}\trans \bX)$ and $\rev{\alpha_0}  \in \mathscr S_{Y|\bX}$. We next assume that there is no monotonic relationship between $\rev{\alpha_0} \trans \bX$ and $\rev{\eta_0}\trans \bX$, and show that this will contradict our condition that $\alpha_0$ is another maximizer of $H(\alpha^\top X, Y)$. Since the Hellinger correlation is a monotone increasing function of the squared Hellinger distance, it suffices to prove that $\mathcal{H}^2(X_1,Y) < \mathcal{H}^2(X_2,Y)$ where $X_1=\rev{\alpha_0} \trans \bX$ and $X_2=\rev{\eta_0}\trans \bX$. Since $X_1$ and $X_2$ do not have a monotonic relationship, the density functions are written as $p(x_1,y)=\int p(x_1|x_2)p(x_2,y)dx_2$ and $p(x_1)=\int p(x_1|x_2)p(x_2) dx_2$. Equipped with this notation, we have
{\allowdisplaybreaks
\begin{align*}
    & \rev{\mathcal H^2}(X_1, Y)\\ 
    &= \iint \rev{\varphi} \left ( \frac{p(x_1,y)}{p(x_1)p(y)} \right ) p(x_1)p(y) dx_1dy\\
    &= \iiint \rev{\varphi} \left ( \frac{\int p(x_1|x_2)p(x_2,y)dx_2}{\int p(x_1|x_2)p(x_2)p(y) dx_2} \right )  \\
    &\quad \quad \quad \times p(x_1|x_2)p(x_2)p(y) dx_1dx_2dy\\
    &\le \iiint \rev{\varphi} \left ( \frac{p(x_2,y)}{p(x_2)p(y)} \right )  p(x_1|x_2)p(x_2)p(y) dx_1dx_2dy,\\
    &=  \iint \rev{\varphi} \left ( \frac{p(x_2,y)}{p(x_2)p(y)} \right ) p(x_2)p(y) dx_2dy\\
    &= \rev{\mathcal H^2}(X_2,Y),
\end{align*}}
where the inequality in the third line comes from Jensen's inequality as used in the proof of Theorem 4 in the appendix of \citet{Kinney2014EquitabilityCoefficient}. 
Notice that  \rev{the the squared Hellinger
distance $\mathcal{H}^2$ uses $\varphi(t) = (t^{1/2}-1)^2$, which is strongly convex.} Thus, the equality holds if and only if $\frac{p(x_2,y)}{p(x_2)p(y)}=1$. In other words, the equality holds if and only if $\rev{\eta_0}\trans \rev{\bX} \indep Y$, which contradicts the assumption that $\rev{\eta_0} \in \mathbb{S}^{p-1}$ is the basis of the central subspace $\mathscr S_{Y|\bX}$. 
Therefore, $\rev{\alpha_0} \trans \rev{\bX}$ and $\rev{\eta_0}\trans \rev{\bX}$ have a monotonic relationship and $\Span(\rev{\eta_0}) = \Span(\rev{\alpha_0} )$.

\end{proof}

\cref{argmin is base of central space} indicates that one can find a basis of the central subspace by optimizing the Hellinger correlation or Bhattacharyya affinity coefficient. 
Since $\rev{\eta_0}$ and $-\rev{\eta_0}$ span the same space, the result of optimization is not unique. 
However, the parameter that one wants to obtain is the central space.
Thus, our goal is identifiable.

The next theorem shows that, for any direction vector, the sample-level Hellinger correlation of our objective function is consistent.
\begin{theorem}\label{convergence in prob}
    Let $\rev{\alpha}  \in \mathbb S^{p-1}$ be an arbitrary vector. Then 
    \begin{align}
        \hat {H}_n(\rev{\alpha} \trans \bX, Y) \overset{P}{\longrightarrow} H(\rev{\alpha} \trans\bX, Y)  
    \end{align}
    where $\overset{P}{\longrightarrow}$ means convergence in probability.
\end{theorem}
\begin{proof}
    Since $H(\rev{\alpha} \trans \bX, Y)$ is a continuous function of $\mathcal B(\rev{\alpha} \trans \bX,Y)$, it suffices to show that $\hat{\mathcal B}_n$ converges to $\mathcal{B}$. 
    Let $\rev{\bU} = (U_{\rev{\alpha} \trans \bX}, U_Y)$ for which it holds that $\mathcal{B} = \mathrm{E}[c^{-1/2}(\rev{\bU})]$. %\ik{check whether $c^{-1/2}$ is correct}
Since $\hat U_i$ converges to $U_i$ in probability, $\hat R_i$ also converges to $R_i$. Then
\begin{equation}\label{tilde B}
    \hat{\mathcal B}_n - \tilde{\mathcal B}_n \overset{P}{\longrightarrow}  0,
\end{equation}
where $\tilde{\mathcal B}_n =\frac{2\sqrt{n - 1}}{n} \sumi R_i $.

\citet{Leonenko2008ADensities} provides an estimator for $\mathrm{E}[f(\bX)^q]$ where $f$ is the density function of $\bX$. 
Our estimator $\hat{\mathcal B}_n$ corresponds to the case where $q = -1/2$. Theorem 3.2 of \citet{Leonenko2008ADensities} shows that $\tilde{\mathcal B}_n$ converges to ${\mathcal B}$ in probability. Thus,
\begin{equation*}
    \hat{\mathcal B}_n - \mathcal B = (\hat{\mathcal B}_n - \tilde{\mathcal B}_n) + (\tilde{\mathcal B}_n - \mathcal B) \overset{P}{\longrightarrow} 0. \tag*{\qedhere}
\end{equation*}
\end{proof}

With the two above theorems, we can show that the sample-level estimator $\rev{\hat{\eta}_n}$ converges to the true SDR direction in probability.

\begin{theorem}\label{thm: consistency}
    Let $\rev{\hat{\eta}_n} = \argmax\{\hat H_n(\rev{\alpha} \trans X, Y) \,|\, \rev{\alpha}  \in \mathbb{S}^{p-1}\}$ and $\rev{\eta_0} = \mathbb{S}^{p-1}$ be the basis for the central subspace $\mathscr S_{Y|\bX}$. Then $z\rev{\hat{\eta}_n} \overset{P}{\longrightarrow}\rev{\eta_0}$ where $|z| = 1$.
\end{theorem}
\begin{proof}
    Suppose that $\rev{\hat{\eta}_n}$ is not a consistent estimator of $\rev{\eta_0}$. 
    Since $\mathbb{S}^{p -1}$ is a compact set, $\{\rev{\hat{\eta}_n}\}$ has a subsequence $\{\rev{\hat{\eta}_{m(n)}}\}$ that converges to $\rev{\eta_*}$ where $\Span(\rev{\eta_*}) \ne \Span(\rev{\eta_0})$. 
    Then $\hat H_{m(n)}(\rev{\hat{\eta}_{m(n)}}\trans \bX, Y) \ge \hat H_{m(n)}(\rev{\eta_0}\trans\bX, Y)$. 
    If we take a limit on both sides, we obtain 
    \begin{equation*}
        H(\rev{\eta_*}\trans X, Y) \ge H(\rev{\eta_0}\trans X, Y).
    \end{equation*}
    By \cref{argmin is base of central space}, there is a contradiction since $\rev{\eta_0} = \argmax H(\rev{\alpha} \trans X, Y)$. Thus, $\rev{\hat{\eta}_n}$ is a consistent estimator of $\rev{\eta_0}$.
\end{proof}
Like the population-level approach, maximizing the estimate of the Hellinger correlation may give two different results, $\rev{\hat{\eta}_n}$ or $-\rev{\hat{\eta}_n}$.
The role of $z$ is to equalize the direction.
If we focus on the projection matrix, we can check that two direction vectors align with the same subspace.

We emphasize that our results are derived without the assumption $\rev{\eta_0}^\top \bX  \indep   \rev{\eta_1}^\top \bX$ where $\rev{\eta_0}^\top\rev{\eta_1} = 0$, which is required for other SDR methods based on dependence measures. This assumption may hold asymptotically for distributions satisfying certain moment conditions. However, in the case of long-tailed distributions, such as the Cauchy distribution, this assumption may not be valid. For further details, see \citet{Diaconis1984AsymptoticsPursuit}.

\section{Numerical Experiments}
\label{SE : NR}

\begin{table*}[ht]
\begin{center}
\caption{Model I: Mean and standard deviation (with parentheses) of $\Delta(\mathcal S_{True}, \mathcal S_{Estimated})$ over 100 samples of size $n$ when predictors are normal.}
\label{Table : B, normal}
\begin{tabular}{c|cccccccc} 
 &  SIR   &  SIR-HC  &  SAVE   &  SAVE-HC  &  DR   &  DR-HC  &  MAVE   &  MAVE-HC\\ 
\hline 
\multirow{2}{*}{$n = 100$ }  &  0.9546  &  0.037  &  0.5154  &  0.0355  &  0.2995  &  0.0371  &  0.0654  &  0.0379\\ 
 &  \small{(0.0801)}  &  \small{(0.0211)}  &  \small{(0.1712)}  &  \small{(0.0194)}  &  \small{(0.0744)}  &  \small{(0.0202)}  &  \small{(0.0223)}  &  \small{(0.0203)}\\ 
\multirow{2}{*}{$n = 200$ }  &  0.8868  &  0.0261  &  0.2977  &  0.0242  &  0.1961  &  0.0232  &  0.0353  &  0.0257\\ 
 &  \small{(0.1552)}  &  \small{(0.014)}  &  \small{(0.098)}  &  \small{(0.0136)}  &  \small{(0.0534)}  &  \small{(0.0126)}  &  \small{(0.0087)}  &  \small{(0.0109)}\\ 
\multirow{2}{*}{$n = 400$ }  &  0.8793  &  0.0183  &  0.1938  &  0.0177  &  0.1262  &  0.0183  &  0.0205  &  0.0195\\ 
 &  \small{(0.1603)}  &  \small{(0.0082)}  &  \small{(0.0485)}  &  \small{(0.0096)}  &  \small{(0.0278)}  &  \small{(0.009)}  &  \small{(0.0053)}  &  \small{(0.0095)}\\ 
\end{tabular}
\end{center}
\end{table*}

\begin{table*}[ht]
\begin{center}
\caption{Model II: Mean and standard deviation (with parentheses) of $\Delta(\mathcal S_{True}, \mathcal S_{Estimated})$ over 100 samples of size $n$ when predictors are normal.}
\label{Table : C, normal}
\begin{tabular}{c|cccccccc}
 &  SIR   &  SIR-HC  &  SAVE   &  SAVE-HC  &  DR   &  DR-HC  &  MAVE   &  MAVE-HC\\ 
\hline 
\multirow{2}{*}{$n = 100$ }  &  0.1191  &  0.0737  &  0.9937  &  0.1203  &  0.1659  &  0.0732  &  0.0699  &  0.0602\\ 
 &  \small{(0.0365)}  &  \small{(0.0406)}  &  \small{(0.0078)}  &  \small{(0.0836)}  &  \small{(0.0475)}  &  \small{(0.0452)}  &  \small{(0.0209)}  &  \small{(0.0251)}\\ 
\multirow{2}{*}{$n = 200$ }  &  0.0747  &  0.0356  &  0.7806  &  0.0654  &  0.1007  &  0.0328  &  0.0405  &  0.0314\\ 
 &  \small{(0.0203)}  &  \small{(0.0182)}  &  \small{(0.2891)}  &  \small{(0.0576)}  &  \small{(0.0323)}  &  \small{(0.0144)}  &  \small{(0.0104)}  &  \small{(0.0119)}\\ 
\multirow{2}{*}{$n = 400$ }  &  0.0535  &  0.0209  &  0.0691  &  0.023  &  0.0662  &  0.0208  &  0.0261  &  0.0179\\ 
 &  \small{(0.0152)}  &  \small{(0.0092)}  &  \small{(0.0231)}  &  \small{(0.0132)}  &  \small{(0.0161)}  &  \small{(0.0093)}  &  \small{(0.0067)}  &  \small{(0.0071)}\\ 
\end{tabular}
\end{center}
\end{table*}

\begin{table*}[t!]
\begin{center}
\caption{Model III: Mean and standard deviation (with parentheses) of $\Delta(\mathcal S_{True}, \mathcal S_{Estimated})$ over 100 samples of size $n$ when predictors are normal.} \label{Table : D, normal}
\begin{tabular}{c|cccccccc}
 &  SIR   &  SIR-HC  &  SAVE   &  SAVE-HC  &  DR   &  DR-HC  &  MAVE   &  MAVE-HC\\ 
\hline 
\multirow{2}{*}{$n = 100$ }  &  0.2873  &  0.0186  &  0.9663  &  0.0183  &  0.4165  &  0.0172  &  0.0455  &  0.018\\ 
 &  \small{(0.1141)}  &  \small{(0.0102)}  &  \small{(0.06)}  &  \small{(0.0094)}  &  \small{(0.1897)}  &  \small{(0.0089)}  &  \small{(0.0149)}  &  \small{(0.0095)}\\ 
\multirow{2}{*}{$n = 200$ }  &  0.2011  &  0.0135  &  0.9131  &  0.0135  &  0.2606  &  0.0127  &  0.0218  &  0.0144\\ 
 &  \small{(0.0583)}  &  \small{(0.0069)}  &  \small{(0.1254)}  &  \small{(0.0072)}  &  \small{(0.0919)}  &  \small{(0.007)}  &  \small{(0.0063)}  &  \small{(0.0073)}\\ 
\multirow{2}{*}{$n = 400$ }  &  0.1433  &  0.0092  &  0.4004  &  0.0096  &  0.1761  &  0.0085  &  0.012  &  0.0094\\ 
 &  \small{(0.0372)}  &  \small{(0.0053)}  &  \small{(0.2249)}  &  \small{(0.0047)}  &  \small{(0.0513)}  &  \small{(0.0042)}  &  \small{(0.0026)}  &  \small{(0.0046)}\\ 
\end{tabular}
\end{center}
\end{table*}

\begin{figure*}[ht]
    \centering
    \includegraphics[width = 17cm, height = 7cm]{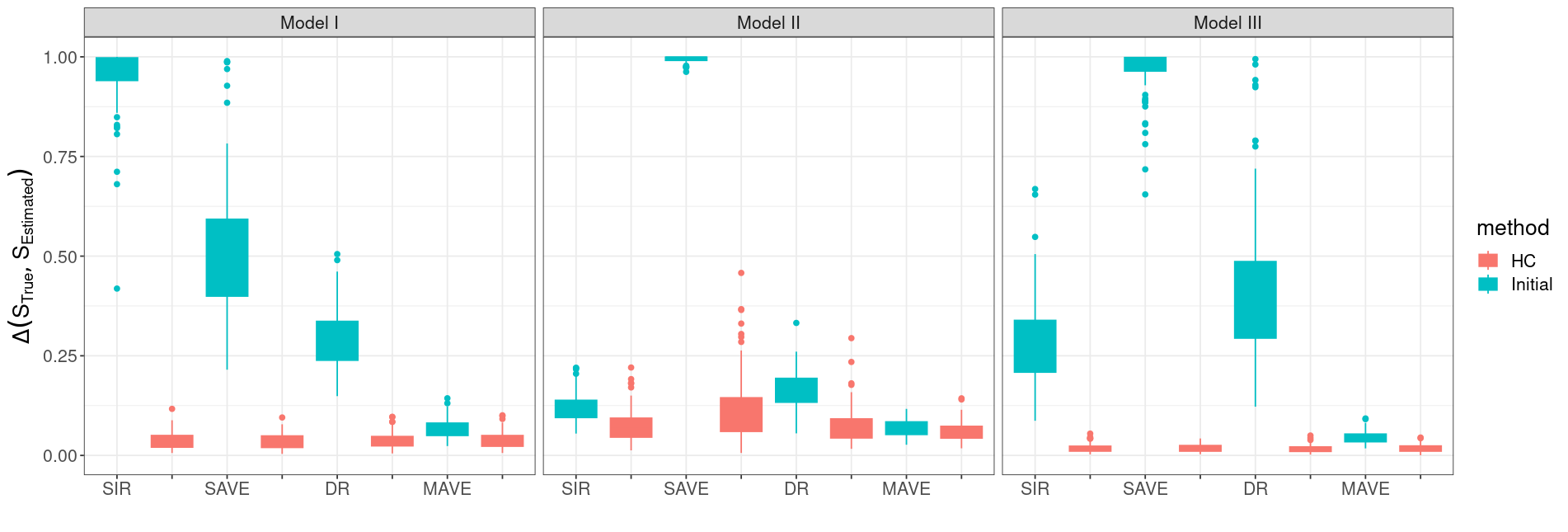}
    \vspace{-0.5cm}
    \caption{Boxplots of $\Delta(\mathcal S_{True}, \mathcal S_{Estimated})$ over 100 samples of size $n=100$ with normal predictors. We compare the performance between SDR Method (aqua blue) and our proposed SDR Method-HC (light coral). As shown, our proposed method consistently outperforms the corresponding SDR methods.}
    \label{figure : Boxplot}
\end{figure*}

To evaluate the accuracy of our proposed method, we conducted simulation experiments under various scenarios:
\begin{align*}
\text{Model I}&:  Y = (\rev{\eta}^\top \bX)^2 + \varepsilon.\\
\text{Model II}&: Y = \exp(\rev{\eta}^\top \bX) + \varepsilon.\\
\text{Model III}&: Y = 5\sin(\rev{\eta}^\top \bX) + \varepsilon,
\end{align*}
where $\varepsilon\sim N(\mu=0,\sigma=0.2)$ and we set $\rev{\eta}$ as
\begin{align*}
\text{Model I}&:  \rev{\eta} = (1,-1,0,0,0,0,0,0,0,0)\trans.\\
\text{Model II}&: \rev{\eta} = (1,1,1,1,1,0,0,0,0,0)\trans.\\
\text{Model III}&: \rev{\eta} = (1,1,0,0,0,0,0,0,0,0)\trans.    
\end{align*}
%To further examine the performance, 
We generate $\bX$ from two different distributions described as
\begin{align*}
    \text{Normal: }& (X_1,\ldots, X_{10})\trans \sim N_{10}(0, I_{10})\\
    \text{Non-normal: }& X_1\sim \mathrm{Exp}(2), X_2 \sim \mathrm{Exp}(4), X_3\sim\chi^2(5), \\
    &X_4 \sim t(15), X_5 \sim t(3),\\
    &(X_6,\ldots,X_{10})\trans \sim N_5 (0,I_5).%  X = (X_1, \ldots, X_{10})
\end{align*}

To assess the performance of our method, we employ the following metric to quantify the difference between two subspaces:
\begin{equation}\label{eq : difference measure}
    \Delta(\mathcal S_{True}, \mathcal S_{Estimated}) = \|P_{\mathcal S_{True}} - P_{\mathcal S_{Estimated}}\|,
\end{equation}
where $\|\cdot\|$ is the maximum eigenvalue of a matrix and $P_{\mathcal S_{True}}$ and $P_{\mathcal S_{Estimated}}$ are the orthogonal projection matrices of the subspace $\mathcal S_{True}=\Span(\rev{\eta_*})$ and $\mathcal S_{Estimated}=\Span(\rev{\hat{\eta}})$.
%The smaller $\Delta$, the more accurate estimate.
A smaller value of $\Delta$ indicates a more accurate estimation.

In addition, to provide a robust comparison of the methods, we generate 100 samples of each case with different sample sizes $n=100,200,400$. Then we compute the mean and standard deviation of $\Delta(\mathcal S_{True}, \mathcal S_{Estimated}) $ over 100 samples and summarize them in Table \ref{Table : B, normal}--Table \ref{Table : D, normal} and also in the appendix.

Table \ref{Table : B, normal} shows the results of the experiment under Model I with the normal predictors. It presents $\Delta(\mathcal S_{True}, \mathcal S_{Estimated})$ when we use SIR \citep{Li1991SlicedReduction}, SAVE \citep{cookSlicedInverseRegression1991}, DR \citep{Li2007OnReduction}, and MAVE \citep{Xia2002AnSpace}. Then SIR-HC, SAVE-HC, DR-HC, MAVE-HC are our methods with their initial values in the iteration as SIR, SAVE, DR, and MAVE, respectively. One can check that SIR fails to recover the central subspace because of the U-shape structure in the model.
However, our proposed method based on SIR successfully discloses the central subspace.
This result shows that our proposed method can overcome the weakness of initial methods. Furthermore, our method enhances the SDR performance significantly even with the MAVE, which is known to be a gold standard. More importantly, the accuracy increases as $n$ increases. It shows an experimental justification of the consistency of our method. 

Table \ref{Table : C, normal} shows the results of the experiment under Model II with the normal predictors. Model II has a strong monotonic relation in which SAVE does not perform well. Similar to Model I, our method can capture central space effectively even with the worst case, and improves its accuracy significantly. 

Table \ref{Table : D, normal} presents the summary of results for Model III with the normal predictors. We can see that the inverse-regression methods such as SIR, SAVE, and DR require a larger sample size to capture the direction correctly. However, with our proposed method, all the estimators become closer to the true direction with high accuracy even with the small sample size. 

Figure \ref{figure : Boxplot} provides the boxplots of  $\Delta(\mathcal S_{True}, \mathcal S_{Estimated})$ for Model I, II, and II with $n=100$ and normal predictors. Overall, our method improves existing SDR methods effectively in various scenarios. %We defer the simulation results for the non-normal predictors to the appendix.

%\rev{Additional results comparing our proposal with contemporary SDR methods based on distance covariance, HSIC, and ball covariance are provided in the appendix. Simulation results using non-sparse $\eta$ and non-normal predictors are also provided in the appendix.}
\rev{Additional comparisons with contemporary SDR methods using distance covariance, HSIC, and ball covariance are provided in the appendix, along with simulation results with non-sparse $\eta$ and non-normal predictors.}

\section{Real Data Analysis}\label{sec: real data}
We apply our methods to the real estate valuation dataset in the UCI Machine Learning Repository \citep{misc_real_estate_valuation_477}. There are 414 observations, and the features in the dataset are
\begin{itemize}
    \item Transaction date,
    \item house age,
    \item distance to the nearest MRT station,
    \item number of convenience stores,
    \item latitude,
    \item longitude, and
    \item (Target) house price of unit area.
\end{itemize}
We remove the transaction date variable and standardize the predictors before applying SDR methods. Subsequently, we randomly divide the dataset into a training sample of size 300 and use the remaining objects as the test sample. SDR methods are then applied to the training set to extract the SDR direction, followed by fitting a local polynomial regression using the remaining variables to predict the house price. \rev{The weights are given equally for each observation and quadratic polynomial was used to fit the model.} Finally, we predict the house price in the test set and compute the test MSE to evaluate the SDR performance. It is important to note that in real data applications, the true central space is unknown, which is why we apply local polynomial regression between the target and the reduced predictor $\rev{\hat{\eta}} \trans X$ to measure the performance. The results are summarized in Table \ref{tab:real data}.

% DCOV, HSIC, BCOV 반영하여 변경 -홍승범
% \begin{table}[h]
%     \centering
%     \caption{Real data analysis: Test MSE of house price with SDR methods}
%     \label{tab:real data}
%     \begin{tabular}{c|cc}
%         \hline
%         SDR Methods & Test MSE: SDR & Test MSE: SDR-HC\\
%         \hline
%          SIR & 0.313 & 0.251\\
%          SAVE & 0.496 & 0.248 \\
%          DR & 0.314 & 0.253\\
%          MAVE & 0.485 & 0.244  \\
%          \hline
%     \end{tabular}
% \end{table}
\begin{table}[h]
    \centering
    \caption{Real data analysis: Test MSE of house price with SDR methods. \rev{The last row is the result from the generalized additive models (GAM) without dimension reduction.}}
    \label{tab:real data}
    \begin{tabular}{ccc}
         \hline
        SDR Methods & MSE: SDR & MSE: SDR-HC\\
         \hline
            SIR	 & 0.313 &	\textbf{0.243}\\
            SAVE & 0.496 &	0.248\\
            DR	 & 0.314 &	0.264\\
            MAVE & 0.485 &	\textbf{0.243}\\
            DCOV & 0.301 &	0.263\\
            HSIC & 0.257 &	0.252\\
            BCOV & 0.279 &	0.260\\ \midrule
            GAM(without SDR)  & 0.245 & -\\
         \hline
    \end{tabular}
\end{table}

Table \ref{tab:real data} shows that all the SDR methods have been improved with our method. The estimated direction with MAVE-HC is 
$$
\rev{\hat{\eta}} =( -0.089 ,-0.987,  0.019,  0.067,  0.117)\trans,
$$
which tells that ``house age'' is a dominating factor in the single-index nonparametric regression model estimation, $\hat Y = \hat f ( \rev{\hat{\eta}} \trans X)$. 

Table \ref{tab:real data} shows that our proposed approach can succinctly capture the essential characteristics of predictors in regression using a single-index model, while maintaining regression performance expressed by MSE.

\section{Discussion}
\label{SE : Discuss}
In this work, we introduce a novel approach to recovering the central space by leveraging the Hellinger correlation, specifically designed for scenarios with a structural dimension of one. Our method sets itself apart from existing approaches, such as those dependent on distance covariance and HSIC, by relaxing the stringent requirements of independence assumptions, which frequently present challenges in practical applications. Significantly, our method excels at deriving theoretical results without imposing such technical constraints. Moreover, numerical experiments demonstrate its capability to enhance current existing sufficient dimension reduction methods. Furthermore, the single-index model SDR provides an interpretation of the intrinsic structure of the nonparametric regression model, as demonstrated in the real-world data analysis.

\rev{Although the current approach has proven effective, it opens up several important avenues for future work. One promising direction is to extend the application of the proposed method to classification problems. This extension would involve modifying the Hellinger correlation to accommodate categorical variables, such as by leveraging the discrete f-divergence \citep{geenens2020copula}. The question of interest is then to see whether our method can help improve classification accuracy while maintaining theoretical validity under weak assumptions. Another direction for future work is to extend our framework to multi-index models by using a sequential generation of SDR directions as done in \citet{christou2020central} or generalizing the Hellinger correlation to multivariate cases. Additionally, it would be valuable to study the convergence rate of the proposed method, which would provide a deeper understanding of its performance across various settings. We believe that exploring these topics will expand the applicability of our method and make valuable contributions to the field.}

\section*{Impact Statement}
This paper presents work whose goal is to advance the field of Machine Learning.  There are many potential societal consequences of our work, none which we feel must be specifically highlighted here.

\clearpage

\section*{Acknowledgements}
We would like to thank four anonymous reviewers and the program chair for the valuable comments and suggestions, which have significantly improved the quality of the paper. For Jun Song, this work is supported by the National Research Foundation of Korea (NRF) grants funded by the Korea government (MSIT) (No. 2022R1C1C1003647, 2022M3J6A1063595, and RS-2023-00219212) and a Korea University Grant (K2402531). Ilmun Kim acknowledges support %from the Basic Science Research Program through 
from the National Research Foundation of Korea (NRF) funded by the Ministry of Education (2022R1A4A1033384), the Korean government (MSIT) (RS-2023-00211073), and support from the Yonsei University Research Fund (2022-22-0289).

\bibliography{EnhancingSDR_HC}
\bibliographystyle{icml2024}

%%%%%%%%%%%%%%%%%%%%%%%%%%%%%%%%%%%%%%%%%%%%%%%%%%%%%%%%%%%%%%%%%%%%%%%%%%%%%%%
%%%%%%%%%%%%%%%%%%%%%%%%%%%%%%%%%%%%%%%%%%%%%%%%%%%%%%%%%%%%%%%%%%%%%%%%%%%%%%%
% APPENDIX
%%%%%%%%%%%%%%%%%%%%%%%%%%%%%%%%%%%%%%%%%%%%%%%%%%%%%%%%%%%%%%%%%%%%%%%%%%%%%%%
%%%%%%%%%%%%%%%%%%%%%%%%%%%%%%%%%%%%%%%%%%%%%%%%%%%%%%%%%%%%%%%%%%%%%%%%%%%%%%%
\newpage
\appendix
\onecolumn
\section{Additional Simulation Results}
The following tables present simulation results when the predictors are non-normal. In most cases, our method significantly enhances existing SDR methods. \rev{The specific simulation settings are provided in Section \ref{SE : NR}.}
%\rev{The true $\eta$s, models and specific distributions are written in Section \ref{SE : NR}}

\begin{table}[h]
\begin{center}
\caption{Model I: Mean and standard deviation (with parentheses) of $\Delta(\mathcal S_{True}, \mathcal S_{Estimated})$ over 100 samples of size $n$ when predictors are non-normal.}
\label{Table : B, non-normal}
\begin{tabular}{c|cccccccc}
 &  SIR   &  SIR-HC  &  SAVE   &  SAVE-HC  &  DR   &  DR-HC  &  MAVE   &  MAVE-HC\\ 
\hline 
\multirow{2}{*}{$n = 100$ }  &  0.6971  &  0.2674  &  0.6735  &  0.2391  &  0.3958  &  0.2381  &  0.3233  &  0.2609\\ 
 &  \small{(0.2136)}  &  \small{(0.2471)}  &  \small{(0.2009)}  &  \small{(0.2264)}  &  \small{(0.1993)}  &  \small{(0.2279)}  &  \small{(0.3243)}  &  \small{(0.2457)}\\ 
\multirow{2}{*}{$n = 200$ }  &  0.7144  &  0.1882  &  0.6396  &  0.1636  &  0.4058  &  0.1804  &  0.1217  &  0.1915\\ 
 &  \small{(0.1396)}  &  \small{(0.1613)}  &  \small{(0.2075)}  &  \small{(0.1299)}  &  \small{(0.1405)}  &  \small{(0.1521)}  &  \small{(0.1517)}  &  \small{(0.1565)}\\ 
\multirow{2}{*}{$n = 400$ }  &  0.7284  &  0.1095  &  0.472  &  0.1171  &  0.3929  &  0.1108  &  0.0533  &  0.1193\\ 
 &  \small{(0.1064)}  &  \small{(0.079)}  &  \small{(0.2394)}  &  \small{(0.0925)}  &  \small{(0.1284)}  &  \small{(0.0851)}  &  \small{(0.0264)}  &  \small{(0.0977)}\\ 
\end{tabular}
\end{center}
\end{table}

\rev{Table \ref{Table : B, non-normal} shows the experimental results under Model I with the non-normal predictors. It presents $\Delta(\mathcal S_{True}, \mathcal S_{Estimated})$ when we use SIR \citep{Li1991SlicedReduction}, SAVE \citep{cookSlicedInverseRegression1991}, DR \citep{Li2007OnReduction}, and MAVE \citep{Xia2002AnSpace}. As in the main text, the SIR-HC, SAVE-HC, DR-HC, MAVE-HC refer to our approaches, using SIR, SAVE, DR, and MAVE, respectively, as their initial values in the iterations. One can observe that SIR fails to recover the central subspace as previously mentioned for cases with normal predictors. In contrast, our proposed method based on SIR successfully discloses the central subspace.}

\begin{table}[h]
\begin{center}
\caption{Model II: Mean and standard deviation (with parentheses) of $\Delta(\mathcal S_{True}, \mathcal S_{Estimated})$ over 100 samples of size $n$ when predictors are non-normal.} \label{Table : C, non-normal}
\begin{tabular}{c|cccccccc}
 &  SIR   &  SIR-HC  &  SAVE   &  SAVE-HC  &  DR   &  DR-HC  &  MAVE   &  MAVE-HC\\ 
\hline 
\multirow{2}{*}{$n = 100$ }  &  0.2287  &  0.0666  &  0.9200  &  0.1190  &  0.4555  &  0.1049  &  0.0686  &  0.0377\\ 
 &  \small{(0.1429)}  &  \small{(0.0679)}  &  \small{(0.0583)}  &  \small{(0.144)}  &  \small{(0.2447)}  &  \small{(0.1307)}  &  \small{(0.0327)}  &  \small{(0.021)}\\ 
\multirow{2}{*}{$n = 200$ }  &  0.1053  &  0.0193  &  0.9067  &  0.0557  &  0.3717  &  0.0425  &  0.0296  &  0.0121\\ 
 &  \small{(0.0476)}  &  \small{(0.0214)}  &  \small{(0.1324)}  &  \small{(0.0965)}  &  \small{(0.2233)}  &  \small{(0.0712)}  &  \small{(0.0125)}  &  \small{(0.0091)}\\ 
\multirow{2}{*}{$n = 400$ }  &  0.0881  &  0.0069  &  0.5365  &  0.0162  &  0.359  &  0.0126  &  0.0148  &  0.004\\ 
 &  \small{(0.0524)}  &  \small{(0.0074)}  &  \small{(0.3664)}  &  \small{(0.0273)}  &  \small{(0.2045)}  &  \small{(0.0243)}  &  \small{(0.0063)}  &  \small{(0.0027)}\\ 
\end{tabular}
\end{center}
\end{table}

\begin{table}[h]
\begin{center}
\caption{Model III: Mean and standard deviation (with parentheses) of $\Delta(\mathcal S_{True}, \mathcal S_{Estimated})$ over 100 samples of size $n$ when predictors are non-normal.} \label{Table : D, non-normal}
\begin{tabular}{c|cccccccc}
&  SIR   &  SIR-HC  &  SAVE   &  SAVE-HC  &  DR   &  DR-HC  &  MAVE   &  MAVE-HC\\ 
\hline 
\multirow{2}{*}{$n = 100$ }  &  0.2457  &  0.0421  &  0.7842  &  0.0393  &  0.4237  &  0.0379  &  0.0286  &  0.0363\\ 
 &  \small{(0.1431)}  &  \small{(0.0286)}  &  \small{(0.2328)}  &  \small{(0.0236)}  &  \small{(0.2585)}  &  \small{(0.0272)}  &  \small{(0.0178)}  &  \small{(0.0251)}\\ 
\multirow{2}{*}{$n = 200$ }  &  0.2044  &  0.0298  &  0.8084  &  0.0292  &  0.3501  &  0.0317  &  0.0178  &  0.0288\\ 
 &  \small{(0.0997)}  &  \small{(0.0235)}  &  \small{(0.1985)}  &  \small{(0.0226)}  &  \small{(0.2204)}  &  \small{(0.0255)}  &  \small{(0.0097)}  &  \small{(0.0225)}\\ 
\multirow{2}{*}{$n = 400$ }  &  0.1889  &  0.0229  &  0.7952  &  0.0237  &  0.3232  &  0.0246  &  0.0117  &  0.0226\\ 
 &  \small{(0.0824)}  &  \small{(0.0186)}  &  \small{(0.1989)}  &  \small{(0.0176)}  &  \small{(0.1856)}  &  \small{(0.0175)}  &  \small{(0.0072)}  &  \small{(0.0175)}\\ 
\end{tabular}
\end{center}
\end{table}
\rev{Table \ref{Table : C, non-normal} and \ref{Table : D, non-normal} show the experimental results under Model II and Model III with the non-normal predictors. Similar to normal cases, our method can effectively capture the central space and significantly improves its accuracy. 
}

\rev{Tables \ref{tab:res_with_dcov_bcov_and_HSIC_via_hcor_Model_I}--\ref{tab:res_with_dcov_bcov_and_HSIC_via_hcor_Model_III} show the experimental results using modern SDR methods as the initial values. We compared three existing methods (via distance covariance \cite{Sheng2013DirectionCovariance}, via HSIC \cite{Zhang2015DirectionCriterion} and via Ball covariance \cite{Zhang2019RobustCovariance}) that capture the central subspace by maximizing dependency measures.
The simulation settings are same as those in Tables 1--3. The predictors follow normal distributions. 
The true $\eta$s and models are detailed in Section \ref{SE : NR}. Our proposed method enhances outcomes by using initial values provided by these three methods. In most cases, the standard deviation has also decreased.}\rev{
One can verify that our method still improves the results, even though the results of the existing method were already promising.
}

\begin{table}[h!]
    \begin{center}
     \caption{Model I: Mean and standard deviation (with parentheses) of $\Delta(\mathcal S_{True}, \mathcal S_{Estimated})$ over 100 samples of size $n$ where modern methods are initial methods.}
     \label{tab:res_with_dcov_bcov_and_HSIC_via_hcor_Model_I}
    \begin{tabular}{c|cccccc}
      & DCOV & DCOV-HC & HSIC & HSIC-HC & BCOV & BCOV-HC\\ \hline
       \multirow{2}{*}{n = 100} &  0.1521  &  0.0373  &  0.1578  &  0.0379  &  0.1425  &  0.0339\\ 
                                 &  \small{(0.1744)}  &  \small{(0.0224)}  &  \small{(0.1694)}  &  \small{(0.0215)}  &  \small{(0.1943)}  &  \small{(0.0191)}\\ 
       \multirow{2}{*}{n = 200} &  0.0966  &  0.0232  &  0.0948  &  0.0224  &  0.0747  &  0.0227\\ 
                                 &  \small{(0.159)}  &  \small{(0.0131)}  &  \small{(0.1293)}  &  \small{(0.0141)}  &  \small{(0.1628)}  &  \small{(0.0121)}\\ 
       \multirow{2}{*}{n = 400} &  0.0732  &  0.0177  &  0.0804  &  0.0182  &  0.0556  &  0.0184\\ 
                                 &  \small{(0.1617)}  &  \small{(0.0086)}  &  \small{(0.1608)}  &  \small{(0.0095)}  &  \small{(0.1654)}  &  \small{(0.0089)}\\ 
    \end{tabular}
    \end{center}
\end{table}

\begin{table}[h!]
    \begin{center}
     \caption{Model II: Mean and standard deviation (with parentheses) of $\Delta(\mathcal S_{True}, \mathcal S_{Estimated})$ over 100 samples of size $n$ where modern methods are initial methods.}
     \label{tab:res_with_dcov_bcov_and_HSIC_via_hcor_Model_II}
    \begin{tabular}{c|cccccc}
      & DCOV & DCOV-HC & HSIC & HSIC-HC & BCOV & BCOV-HC\\ \hline
 \multirow{2}{*}{n = 100} &  0.1632  &  0.0843  &  0.1077  &  0.0680  &  0.1087  &  0.0660\\ 
 &  \small{(0.076)}  &  \small{(0.0539)}  &  \small{(0.0332)}  &  \small{(0.0354)}  &  \small{(0.0422)}  &  \small{(0.0365)}\\ 
 \multirow{2}{*}{n = 200} &  0.0997  &  0.0411  &  0.0743  &  0.0373  &  0.0396  &  0.0307\\ 
 &  \small{(0.0291)}  &  \small{(0.021)}  &  \small{(0.0219)}  &  \small{(0.0173)}  &  \small{(0.0153)}  &  \small{(0.0155)}\\ 
 \multirow{2}{*}{n = 400} &  0.0657  &  0.0219  &  0.0469  &  0.0204  &  0.0289  &  0.0163\\ 
 &  \small{(0.0549)}  &  \small{(0.011)}  &  \small{(0.0131)}  &  \small{(0.0089)}  &  \small{(0.0983)}  &  \small{(0.007)}\\ 
    \end{tabular}
    \end{center}
    
\end{table}

\begin{table}[h!]
    \begin{center}
     \caption{Model III: Mean and standard deviation (with parentheses) of $\Delta(\mathcal S_{True}, \mathcal S_{Estimated})$ over 100 samples of size $n$ where modern methods are initial methods.}
     \label{tab:res_with_dcov_bcov_and_HSIC_via_hcor_Model_III}
    \begin{tabular}{c|cccccc}
      & DCOV & DCOV-HC & HSIC & HSIC-HC & BCOV & BCOV-HC\\ \hline
 \multirow{2}{*}{n = 100}&  0.1459  &  0.0194  &  0.1223  &  0.0177  &  0.0594  &  0.0201\\ 
 &  \small{(0.0452)}  &  \small{(0.0103)}  &  \small{(0.0399)}  &  \small{(0.0104)}  &  \small{(0.0284)}  &  \small{(0.0116)}\\ 
 \multirow{2}{*}{n = 200} &  0.0873  &  0.0126  &  0.075  &  0.0114  &  0.0204  &  0.0136\\ 
 &  \small{(0.0252)}  &  \small{(0.0067)}  &  \small{(0.0201)}  &  \small{(0.0066)}  &  \small{(0.007)}  &  \small{(0.006)}\\ 
 \multirow{2}{*}{n = 400} &  0.0594  &  0.0087  &  0.0503  &  0.0090  &  0.0114  &  0.009\\ 
 &  \small{(0.0143)}  &  \small{(0.0039)}  &  \small{(0.0127)}  &  \small{(0.0043)}  &  \small{(0.0035)}  &  \small{(0.0041)}\\ 
    \end{tabular}
    \end{center}
\end{table}

\rev{Tables \ref{tab:non-sparse_Model_I}--\ref{tab:non-sparse_Model_III} show the experimental results when $\eta$ is not sparse.
The sparsity of the direction vector does not affect our proposed method as well as the other existing methods.
We changed only the true $\eta$s as follows:
\begin{align*}
\text{Model I}&:  \eta = (1,1,1,1,1,1,1,1,1,1)\trans/\sqrt{10}.\\
\text{Model II}&: \eta = (1,1,1,-1,-1,-1,-1,1,1,-1)\trans/\sqrt{10}.\\
\text{Model III}&: \eta = (3,-1,4,-2,-4,5,1,-3,-5,2)\trans/\sqrt{110}.    
\end{align*}
All other conditions, including the models and the distribution of predictors, remain the same as described in Section \ref{SE : NR}.
}

\begin{table}[h!]
    \begin{center}
    \caption{Model I: Mean and standard deviation (with parentheses) of $\Delta(\mathcal S_{True}, \mathcal S_{Estimated})$ over 100 samples of size $n$ when true $\eta$ is non-sparse.}
       \label{tab:non-sparse_Model_I}
    \begin{tabular}{c|cccccccc}
 &  SIR  &  SIR-HC  &  SAVE  &  SAVE-HC  &  DR  &  DR-HC  &  MAVE  &  MAVE-HC\\ 
\hline 
 \multirow{2}{*}{n = 100}&  0.9462  &  0.8865  &  0.5093  &  0.6764  &  0.2946  &  0.2274  &  0.0534  &  0.0328\\ 
 &  \small{(0.0903)}  &  \small{(0.1755)}  &  \small{(0.185)}  &  \small{(0.3597)}  &  \small{(0.0766)}  &  \small{(0.2351)}  &  \small{(0.015)}  &  \small{(0.0112)}\\ 
 \multirow{2}{*}{n = 200}&  0.8857  &  0.8435  &  0.3138  &  0.1749  &  0.1821  &  0.0678  &  0.0246  &  0.0152\\ 
 &  \small{(0.154)}  &  \small{(0.1636)}  &  \small{(0.1186)}  &  \small{(0.2072)}  &  \small{(0.0442)}  &  \small{(0.0471)}  &  \small{(0.0061)}  &  \small{(0.0046)}\\ 
 \multirow{2}{*}{n = 400}&  0.8804  &  0.6878  &  0.1819  &  0.0564  &  0.1249  &  0.0403  &  0.013  &  0.0085\\ 
 &  \small{(0.1703)}  &  \small{(0.3057)}  &  \small{(0.0538)}  &  \small{(0.0415)}  &  \small{(0.0349)}  &  \small{(0.0325)}  &  \small{(0.0032)}  &  \small{(0.0026)}\\ 
    \end{tabular}
    \end{center}
\end{table}

\begin{table}[h!]
    \begin{center}
    \caption{Model II: Mean and standard deviation (with parentheses) of $\Delta(\mathcal S_{True}, \mathcal S_{Estimated})$ over 100 samples of size $n$ when true $\eta$ is non-sparse.}
        \label{tab:non-sparse_Model_II}
    \begin{tabular}{c|cccccccc}
&  SIR  &  SIR-HC  &  SAVE  &  SAVE-HC  &  DR  &  DR-HC  &  MAVE  &  MAVE-HC\\ 
\hline 
\multirow{2}{*}{n = 100} &  0.1554  &  0.1223  &  0.9922  &  0.7911  &  0.5444  &  0.5769  &  0.0725  &  0.0741\\ 
 &  \small{(0.0431)}  &  \small{(0.0496)}  &  \small{(0.0095)}  &  \small{(0.1455)}  &  \small{(0.1962)}  &  \small{(0.3026)}  &  \small{(0.0259)}  &  \small{(0.0358)}\\ 
\multirow{2}{*}{n = 200} &  0.0992  &  0.0677  &  0.8597  &  0.6362  &  0.4024  &  0.2929  &  0.0427  &  0.0372\\ 
 &  \small{(0.0272)}  &  \small{(0.0307)}  &  \small{(0.2312)}  &  \small{(0.255)}  &  \small{(0.1727)}  &  \small{(0.2632)}  &  \small{(0.0137)}  &  \small{(0.0194)}\\ 
\multirow{2}{*}{n = 400}  &  0.0519  &  0.0623  &  0.0696  &  0.0679  &  0.0704  &  0.0691  &  0.0378  &  0.0594\\ 
 &  \small{(0.0109)}  &  \small{(0.0189)}  &  \small{(0.0181)}  &  \small{(0.0228)}  &  \small{(0.0168)}  &  \small{(0.0194)}  &  \small{(0.0091)}  &  \small{(0.0179)}\\ 
    \end{tabular}
    \end{center}
\end{table}

\begin{table}[h!]
    \begin{center}
    \caption{Model III: Mean and standard deviation (with parentheses) of $\Delta(\mathcal S_{True}, \mathcal S_{Estimated})$ over 100 samples of size $n$ when true $\eta$ is non-sparse.}
    \label{tab:non-sparse_Model_III}
    \begin{tabular}{c|cccccccc}
 &  SIR  &  SIR-HC  &  SAVE  &  SAVE-HC  &  DR  &  DR-HC  &  MAVE  &  MAVE-HC\\ 
\hline 
\multirow{2}{*}{n = 100} &  0.1405  &  0.0871  &  0.9903  &  0.763  &  0.194  &  0.1012  &  0.0371  &  0.0398\\ 
 &  \small{(0.0568)}  &  \small{(0.0497)}  &  \small{(0.0133)}  &  \small{(0.1594)}  &  \small{(0.0831)}  &  \small{(0.098)}  &  \small{(0.0131)}  &  \small{(0.0111)}\\ 
\multirow{2}{*}{n = 200} &  0.0853  &  0.0409  &  0.8053  &  0.5614  &  0.1198  &  0.0542  &  0.0176  &  0.0243\\ 
 &  \small{(0.0289)}  &  \small{(0.018)}  &  \small{(0.2765)}  &  \small{(0.256)}  &  \small{(0.0417)}  &  \small{(0.0273)}  &  \small{(0.0047)}  &  \small{(0.0073)}\\ 
\multirow{2}{*}{n = 400} &  0.0624  &  0.03  &  0.0893  &  0.0373  &  0.0808  &  0.0338  &  0.0109  &  0.0174\\ 
 &  \small{(0.0201)}  &  \small{(0.0116)}  &  \small{(0.0334)}  &  \small{(0.0201)}  &  \small{(0.0244)}  &  \small{(0.0168)}  &  \small{(0.0029)}  &  \small{(0.0057)}\\ 
    \end{tabular}
    \end{center}
\end{table}

\rev{The overall behavior of the experimental results with non-sparse direction vectors is not significantly different from those in Tables \ref{Table : B, normal}--\ref{Table : D, normal}. We observe that, as the sample size increases, our proposed method detects the true directions with high accuracy.
}

%%%%%%%%%%%%%%%%%%%%%%%%%%%%%%%%%%%%%%%%%%%%%%%%%%%%%%%%%%%%%%%%%%%%%%%%%%%%%%%
%%%%%%%%%%%%%%%%%%%%%%%%%%%%%%%%%%%%%%%%%%%%%%%%%%%%%%%%%%%%%%%%%%%%%%%%%%%%%%%

\end{document}